%% file: main.tex
\documentclass{article}

\PassOptionsToPackage{numbers, compress}{natbib}

\usepackage[preprint]{neurips_2026}

\usepackage[utf8]{inputenc}
\usepackage[T1]{fontenc}
\usepackage{url}
\usepackage{booktabs}
\usepackage{amsfonts}
\usepackage{amsmath}
\usepackage{amssymb}
\usepackage{amsthm}
\usepackage{nicefrac}
\usepackage{microtype}
\usepackage{graphicx}
\usepackage{xcolor}
\usepackage{hyperref}   

\hypersetup{
    colorlinks=true,
    linkcolor=blue!80,
    citecolor=blue!80,
    urlcolor=blue!80
}

\newtheorem{proposition}{Proposition}
\newcommand{\KL}{\mathrm{KL}}
\newcommand{\Ren}{D^{\mathrm{R}}}
\newcommand{\Tsa}{D^{\mathrm{T}}}
\newcommand{\Norm}{\mathcal{N}}

\title{Generalized Graph Variational Autoencoders:\\Bounded Divergences Control Posterior Collapse}

\author{%
  Kleyton da Costa\thanks{Corresponding author. This work was developed during the author's M.Sc in Computer Science at Pontifical Catholic University of Rio de Janeiro, Brazil.} \\
  UCL \& Holistic AI \\
  London, UK \\
  \texttt{kleyton.costa.25@ucl.ac.uk} \\
  \And
  Bernardo Modenesi\\
  University of Utah\\
  Salt Lake City, US \\
  \texttt{bernardo.modenesi@utah.edu} \\
  \AND
  Ivan F.M. Menezes\\
  PUC-Rio\\
  Rio de Janeiro, Brazil \\
  \texttt{ivan@puc-rio.br} \\
  \And
  Hélio Lopes\\
  PUC-Rio\\
  Rio de Janeiro, Brazil \\
  \texttt{lopes@inf.puc-rio.br} \\
}

\begin{document}

\maketitle

\begin{center}
\vspace{-0.6em}
\small
\textbf{Code:} \href{https://github.com/kleyt0n/ggva}{https://github.com/kleyt0n/ggva}
\vspace{0.6em}
\end{center}

\begin{abstract}
The variational graph autoencoder (VGAE) regularizes its posterior toward the prior with the Kullback-Leibler divergence, a choice inherited from the variational autoencoder rather than argued for. We introduce the \emph{generalized graph variational autoencoder} (GGVA), which replaces that term with any member of the R\'enyi-Tsallis family of order $q$ while leaving every other part of the model untouched. Both members admit closed forms for diagonal Gaussians and both recover the KL exactly as $q \to 1$, so the VGAE is the $q=1$ arm of our own model rather than a separate baseline, and any measured difference is attributable to a single scalar. Our analysis identifies boundedness, not the order, as the operative property: for $q<1$ the Tsallis divergence is bounded above by $1/(1-q)$, independently of the latent width, whereas the KL and the R\'enyi divergence of the same order are unbounded. On ten graphs spanning three synthetic families, a social network, three citation networks, a
connectome, a power grid and a road network, $q$ moves the retained posterior information by up to $49\times$ relative to the VGAE, while the R\'enyi arm at
the \emph{same} order stays within $1.02$--$1.30\times$ of it on all six larger real graphs (isolating the bound as the cause). The retained information is usable: probing the frozen embedding for node class, a label absent from the objective, gives GGVA up to $+0.14$ macro-F1 over the VGAE on CiteSeer, with the R\'enyi control again
tracking the VGAE. We also report what the design was built to expose: none of this reaches held-out link-prediction accuracy on any of the six larger real graphs, and
boundedness delays posterior collapse rather than preventing it.
\end{abstract}

\section{Introduction}

The variational graph autoencoder \citep{kipf2016variational} embeds nodes by
maximizing an evidence lower bound whose second term penalizes the approximate
posterior for departing from a standard normal prior. That penalty is the
Kullback-Leibler divergence. The choice is almost never examined: it is
inherited from the variational autoencoder \citep{kingma2014auto}, where it is
convenient because it has a closed form for Gaussians, and it is retained in the
graph setting for the same reason.

Convenience is a weak argument, because an entire one-parameter family of
divergences has closed forms for Gaussians. The R\'enyi
\citep{renyi1961measures} and Tsallis \citep{tsallis1988possible} divergences of
order $q$ both contract to the KL as $q \to 1$, and substituting one for the
other changes the geometry of the penalty without changing the model, the
encoder, the decoder, or the reconstruction term. The same order $q$ has been
used as a tunable inductive bias elsewhere in deep learning, from sparse
attention to loss design \citep{dacosta2026perspectives}.

There is a specific reason to expect this to matter. The KL is
\emph{unbounded}: weighted heavily enough it can always reduce the objective
further by pulling the posterior onto the prior, which is the mechanism behind
posterior collapse \citep{bowman2016generating,sonderby2016ladder,alemi2018fixing}.
The Tsallis divergence of order $q<1$ is bounded above by $1/(1-q)$, and a
saturated penalty has no gradient left with which to pull.

Graphs make the scale mismatch especially visible: reconstruction is averaged
over edges while the latent penalty is averaged over nodes, so the same nominal
weight exerts different pressure as graph size and density change. Sweeping that
pressure over graphs of very different sizes provides a direct stress test of
whether a divergence merely rescales the KL or changes the response itself.

We contribute (i) the GGVA (\S\ref{sec:method}), which parameterizes the latent penalty by $q$ with closed forms for diagonal Gaussians and an exact reduction to
the VGAE at $q=1$, making the comparison controlled by construction; (ii) the $1/(1-q)$ bound
(Proposition~\ref{prop:bound}) together with an experimental design that
\emph{isolates} it, since the R\'enyi and Tsallis members share the order, the
escort integral and every Gaussian term and differ only in whether the result is
bounded (\S\ref{sec:protocol}); and (iii) two negative results, reported with the controls that
establish them. The effect does not reach link-prediction accuracy on any of the six larger real
graph, and boundedness delays collapse rather than preventing it.

\section{Related work}
\label{sec:related}

Posterior collapse and over-pruning are established failure modes of variational
autoencoders \citep{bowman2016generating,burda2016importance,alemi2018fixing}.
Epitomic VGAE addresses inactive units in this same graph setting by introducing
competing groups of latent variables \citep{khan2021epitomic}. Our question is
complementary: we leave the architecture, prior and training protocol fixed and
change only the scalar divergence used for the nodewise posterior penalty.

The closest precedent for that mechanism is q-VAE, which derives a
Tsallis-deformed lower bound for disentangled representation learning
\citep{kobayashi2020qvae}. GGVA instead keeps the reconstruction term and
Gaussian model fixed, swaps only the penalty, and uses a same-order R\'enyi arm
to isolate boundedness. R\'enyi variational inference changes the bound itself
\citep{li2016renyi}; InfoVAE and WAE match the aggregated posterior, including
with kernel MMD penalties \citep{zhao2019infovae,tolstikhin2018wasserstein}; and
$t^3$VAE couples a power divergence to heavy-tailed priors, encoders and decoders
\citep{kim2024t3vae}. These methods motivate alternatives to the KL, but none is
the controlled one-scalar comparison studied here.

\section{The generalized graph variational autoencoder}
\label{sec:method}

\subsection{Background}

Let $G=(V,E)$ with $N=|V|$, adjacency $A$ and node features $X$. An encoder maps
$(X,A)$ to a factorized Gaussian posterior
$q_\phi(Z \mid X, A) = \prod_{i=1}^{N} \Norm\!\left(z_i \mid \mu_i,
\mathrm{diag}(\sigma_i^2)\right)$
with prior $p(Z)=\prod_i \Norm(z_i \mid 0, I_d)$, and an inner-product decoder
scores a pair as $p(A_{ij}=1 \mid Z) = \mathrm{sigmoid}(z_i^\top z_j)$. Training
minimizes
\begin{equation}
  \mathcal{L}
  \;=\;
  \underbrace{\mathcal{L}_{\text{rec}}}_{\text{per edge}}
  \;+\;
  \beta\,\lambda\,
  \underbrace{\frac{1}{N}\sum_{i=1}^{N} D\!\left(q_\phi(z_i \mid X,A)\,\|\,p(z_i)\right)}_{\text{per node}},
  \label{eq:objective}
\end{equation}
where $\mathcal{L}_{\text{rec}}$ is a binary cross-entropy over observed edges and
sampled non-edges, $\beta$ is the $\beta$-VAE weight \citep{higgins2017beta}, and
$\lambda$ is a commensurability factor. The two terms are not on the same scale:
the divergence is averaged over \emph{nodes} and summed over latent dimensions,
while the reconstruction term is averaged over \emph{edges}. The reference
implementation sets $\lambda = 1/N$; we write $\lambda = c/N$ and treat $c$ as the
knob that controls how hard the penalty presses. Setting $D=\KL$ recovers the
VGAE exactly. Only the product $\beta\lambda$ is identifiable in
\eqref{eq:objective}; we fix $\beta=1$ throughout, so $c$ is the sole weight
parameter.

\subsection{A one-parameter family of latent penalties}

We take the latent penalty from the R\'enyi--Tsallis family. Both members are
transforms of one quantity, the \emph{escort integral}
$I_q(p\|r) = \int p(z)^{q} r(z)^{1-q}\,\mathrm{d}z$ of order $q>0$ --- the
logarithmic and the linear transform respectively:
\begin{equation}
  \Ren_q \;=\; \frac{\log I_q}{q-1},
  \qquad
  \Tsa_q \;=\; \frac{I_q - 1}{q-1}
           \;=\; \frac{e^{(q-1)\Ren_q}-1}{q-1},
  \qquad
  \lim_{q \to 1}\Ren_q \;=\; \lim_{q \to 1}\Tsa_q \;=\; \KL .
  \label{eq:family}
\end{equation}
For a one-dimensional
Gaussian posterior $\Norm(m,\sigma^2)$ against a standard normal prior, writing
$s = (1-q)\sigma^2 + q$, the R\'enyi divergence is available in closed form
\citep{vanerven2014renyi},
\begin{equation}
  \Ren_q \;=\; \frac{q}{2}\cdot\frac{m^{2}}{s}
              \;-\; \frac{\log s}{2(q-1)}
              \;-\; \log \sigma ,
  \qquad s > 0 .
  \label{eq:renyi-gauss}
\end{equation}
It is additive across latent dimensions, so the $d$-dimensional value is the sum
of \eqref{eq:renyi-gauss}. The Tsallis divergence of the joint posterior is
\emph{not} additive: exponentiating the summed R\'enyi divergence gives
\begin{equation}
  \Tsa_q\!\left(q_\phi(z_i) \,\|\, p(z_i)\right)
  \;=\;
  \frac{\exp\!\Big((q-1)\textstyle\sum_{k=1}^{d}{\Ren_q}^{(k)}\Big) - 1}{q-1},
  \label{eq:tsallis-joint}
\end{equation}
which is pseudo-additive,
$\Tsa_q(P_1 \otimes P_2) = \Tsa_q(P_1) + \Tsa_q(P_2)
+ (q-1)\Tsa_q(P_1)\Tsa_q(P_2)$.
This non-additivity is what makes the bound below independent of $d$. We
implement \eqref{eq:tsallis-joint} with $\mathrm{expm1}$ and $\mathrm{log1p}$
throughout; the naive form loses all significant digits near $q=1$, which is
precisely the regime of interest (Appendix~\ref{app:numerics}).

\subsection{Boundedness}

\begin{proposition}[Bounded latent penalty]
\label{prop:bound}
For $0 < q < 1$ and any posterior, $0 \le \Tsa_q \le \dfrac{1}{1-q}$, with the
bound independent of the latent dimension $d$. For $q \ge 1$, $\Tsa_q$ is
unbounded, as is $\Ren_q$ for every $q>0$. For the Gaussian pair in
\eqref{eq:renyi-gauss} and $q>1$, both are infinite whenever
$\sigma^2 \ge q/(q-1)$.
\end{proposition}

\begin{proof}
For $q \in (0,1)$ H\"older's inequality gives $I_q \le 1$, so
$I_q-1 \in [-1,0]$ and dividing by $q-1<0$ yields the bound; it does not involve
$d$ because \eqref{eq:tsallis-joint} transforms the \emph{summed} R\'enyi
divergence rather than summing per-dimension transforms. For unit-variance
Gaussians separated by a mean $m$, $\Ren_q=q m^2/2$, which is unbounded as
$|m|\to\infty$ for every $q>0$. Equation~\eqref{eq:family} then makes $\Tsa_q$
unbounded for $q>1$, while $q=1$ is the unbounded KL. Finally, the Gaussian
convergence condition follows from $s=(1-q)\sigma^2+q>0$ in
\eqref{eq:renyi-gauss}.
\end{proof}

The consequence is about optimization rather than about ranges. A penalty that
saturates has a vanishing gradient once the posterior is far enough from the
prior, so it cannot drag the posterior back however heavily it is weighted. The
KL can, and does. Crucially, $\Ren_q$ for $q<1$ is a \emph{softened} KL that
remains unbounded, which gives us the control we need: if the bounded and
unbounded members of the same family at the same order behave differently, the
boundedness caused it.

\subsection{Exact reduction to the VGAE}

At $q=1$ both expressions in \eqref{eq:family} are singular while their limit is
finite, so we dispatch to the closed-form KL when $|q-1|<10^{-8}$: the
GGVA at $q=1$ evaluates the identical tensor the VGAE evaluates. Training
\emph{trajectories} cannot be compared bit-for-bit, since sparse message passing
reduces nondeterministically and two VGAE runs already disagree in the eighth
significant figure, so we assert the meaningful claim instead --- swapping the
VGAE for GGVA($q{=}1$) must perturb a run no more than re-running the VGAE does.
Measured over 20 epochs, the cross-model deviation and same-model floor agree to
two significant figures, both $9.5\times10^{-7}$.

\section{Experimental protocol}
\label{sec:protocol}

A divergence swap changes the \emph{shape} of the penalty and its
\emph{magnitude} at once, and only the first is interesting. Three controls
separate them.

\textbf{One knob.} Every arm is a GGVA built by the same call, with the same
encoder (a two-layer variational GCN, hidden $32$, latent $d=16$), decoder,
split, seed, epoch budget and optimizer. The VGAE baseline is the $q=1$ arm, so
even the model class is held constant.

\textbf{Magnitude is swept, not assumed away.} Each arm is run over
$c \in \{1,10,10^2,10^3,10^4\}$, five decades of penalty weight. A family that
were merely a rescaled KL would trace the KL's curve shifted sideways; a family
that differs in kind traces a differently shaped curve.

\textbf{Bounded against unbounded at fixed $q$.} We run Tsallis and R\'enyi at
$q=0.5$, which share everything but the bound.

\textbf{Metrics and data.} Held-out AUC on an $85/5/10$ edge split, read at the
epoch maximizing validation AUC. Collapse is reported in \emph{KL nats for every
arm} whatever it trained against, since the arms differ in their own units; we
also report active latent units and mean posterior standard deviation. Ten
graphs, five seeds, split redrawn per seed, CPU only: three synthetic (a
$300$-node stochastic block model \citep{holland1983stochastic},
Barab\'asi--Albert \citep{barabasi1999emergence}, and an Erd\H{o}s--R\'enyi
\emph{negative control} on which AUC must sit at chance), Zachary's karate club,
the citation networks Cora, CiteSeer and PubMed \citep{sen2008collective}, and
three non-citation real graphs (the \emph{C.\ elegans} connectome, the US power
grid \citep{watts1998collective}, Euroroad) whose lack of informative features
and different degree distributions stop the result from being a claim about
citation graphs. Full details in Appendix~\ref{app:setup}.
Aggregates over the \emph{six larger real graphs} mean Cora, CiteSeer, PubMed,
\emph{C.\ elegans}, the power grid and Euroroad. Karate remains in every
per-graph result but is excluded from aggregates because its 10\% test split has
only about eight edges, making its AUC exceptionally unstable.

\section{Results}

\begin{figure}[t]
  \centering
  \includegraphics[width=0.93\textwidth]{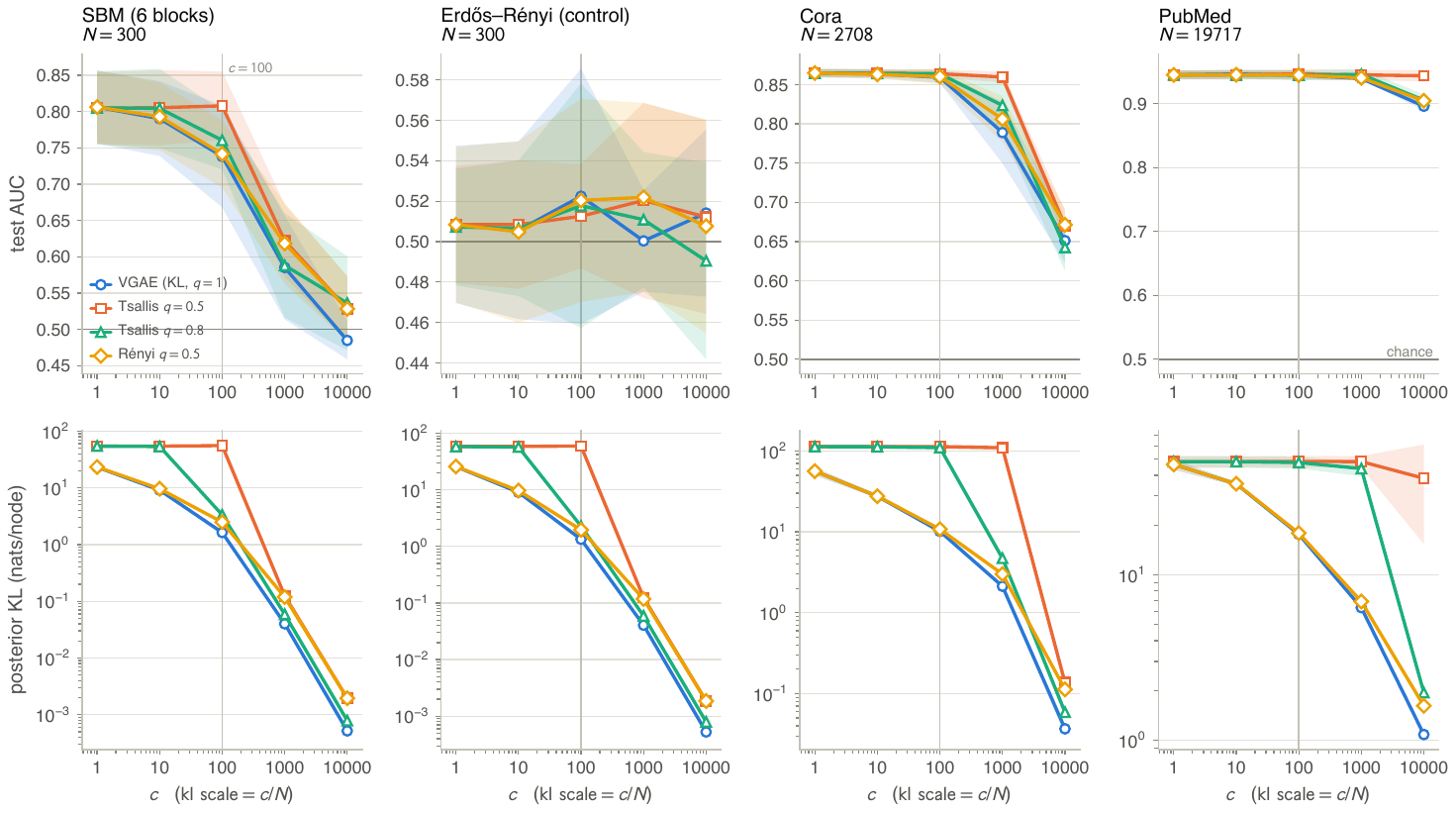}
  \caption{\textbf{The shape of the response differs, not just its scale.}
  Held-out AUC (top) and retained posterior information in KL nats (bottom)
  against the penalty weight $c$, for four objectives. The VGAE, Tsallis
  $q{=}0.8$ and R\'enyi $q{=}0.5$ decay steadily; Tsallis $q{=}0.5$ holds a rich
  posterior across two decades and then falls off a cliff. A pure magnitude
  effect would shift a curve sideways, not change its shape. Bands are 95\%
  $t$-intervals over five seeds. Four of the ten graphs are shown; the rest are
  in Appendix~\ref{app:full}.}
  \label{fig:sweep}
\end{figure}

\subsection{\texorpdfstring{Retained posterior information is ordered by $q$}{Retained posterior information is ordered by q}}

Table~\ref{tab:collapse} reports the posterior state at $c=100$. Retained
information is ordered by $q$ on every one of the ten graphs, with $q{=}1.5$ ---
an \emph{unbounded}, steeper member --- consistently below the VGAE. On the SBM
the VGAE has lost $40\%$ of its latent units at this weight while $q{=}0.5$ has
lost none. Mean posterior standard deviation gives the same ladder in the
opposite direction: $0.026$, $0.634$, $0.715$, $0.729$ and $0.815$ for Tsallis
$q\in\{0.5,0.8,0.95\}$, the VGAE and Tsallis $q{=}1.5$, respectively.

The sweep (Figure~\ref{fig:sweep}) separates this from a scaling artifact. If
$q$ merely rescaled the penalty, the $q{=}0.5$ curve would be the VGAE's
translated along $c$. Instead it is flat where the VGAE is already decaying, and
then drops abruptly --- a different functional shape.

\subsection{Boundedness, not the order, is the mechanism}

The R\'enyi row of Table~\ref{tab:collapse} is the control. At $q=0.5$ it shares
the order, the escort integral and every Gaussian term with the Tsallis arm, and
differs only in being unbounded. Across the six larger real graphs it retains
$1.02$--$1.30\times$ the VGAE's posterior information, while Tsallis at the same
order retains $2.8$--$26.6\times$. The order alone does almost nothing; the bound
does the work.

The Erd\H{o}s--R\'enyi negative control separates retained information from
useful signal. At $c=100$, Tsallis $q{=}0.5$ retains $59.44$ KL nats per node,
$44.4\times$ the VGAE's $1.34$, yet its AUC is $0.513$ and its random-label
probe remains at chance. A preserved posterior is therefore not, by itself,
evidence of learned structure.

\subsection{Accuracy does not separate the arms}

At $c=100$, selecting the best of four prespecified Tsallis orders gives
$+0.070$ AUC over the VGAE on the SBM and $+0.224$ on karate, but at most
$+0.014$ on the six larger real graphs (Euroroad $+0.014$, CiteSeer $+0.007$,
Cora and \emph{C.\ elegans} $+0.005$, the power grid $+0.003$, and PubMed
$+0.002$). Such maxima are upward-biased: the corresponding maximum on the
Erd\H{o}s--R\'enyi negative control is itself $+0.005$. The highest-mean
prespecified Tsallis arm scores $0.7789$ across the six larger real graphs
against the VGAE's $0.7764$, and a Friedman test over all ten tasks and six arms
does not reject ($\chi^2 = 7.26$, $p = 0.20$; mean-rank spread $1.70$ inside the
Nemenyi critical difference of $2.384$). \textbf{On link-prediction accuracy,
choosing $q$ buys nothing we can measure.}

\begin{figure}[t]
  \centering
  \includegraphics[width=0.82\textwidth]{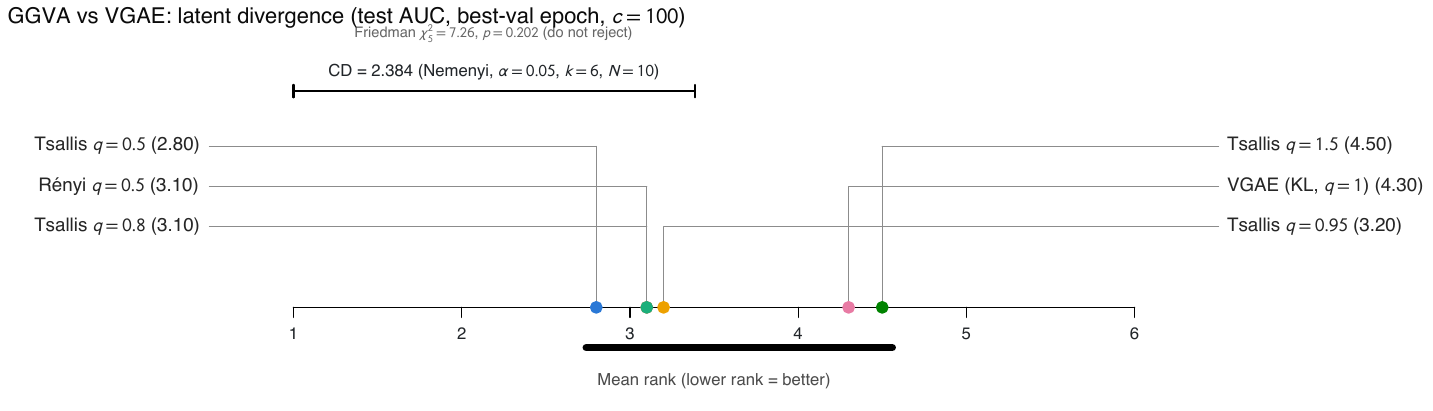}
  \caption{\textbf{No arm separates in aggregate link-prediction accuracy.}
  Critical-difference diagram for six arms over ten graphs. All mean ranks lie
  in one Nemenyi clique ($\mathrm{CD}=2.384$, $\alpha=0.05$).}
  \label{fig:cd}
\end{figure}

\begin{table}[t]
  \caption{\textbf{Retained posterior information is ordered by $q$; the R\'enyi
  control is not.} Posterior KL in nats per node at $c=100$, mean over five
  seeds, measured with the KL for \emph{every} arm regardless of its training
  objective. ``Bound'' is Proposition~\ref{prop:bound}. The last column is
  held-out AUC averaged over the six larger real graphs. Tsallis $q{=}0.5$ retains
  $2.8$--$49\times$ the VGAE's information across all ten graphs, while R\'enyi
  at the \emph{same order} stays within $1.02$--$1.30\times$ of it on the six larger real
  graphs; accuracy is flat throughout. All ten graphs in Appendix~\ref{app:full}.}
  \label{tab:collapse}
  \centering
  \small
  \begin{tabular}{lrrrrr}
    \toprule
    & & \multicolumn{3}{c}{Posterior KL (nats/node)} & Mean AUC \\
    \cmidrule(lr){3-5}
    Objective & Bound & SBM & Cora & PubMed & (6 larger real) \\
    \midrule
    VGAE (KL, $q{=}1$)      & $\infty$ &  1.63 &  10.13 & 17.59 & 0.7764 \\
    Tsallis $q{=}0.5$       & 2.00     & \textbf{55.50} & \textbf{112.99} & \textbf{48.46} & 0.7789 \\
    Tsallis $q{=}0.8$       & 5.00     &  3.34 & 109.83 & 47.74 & 0.7788 \\
    Tsallis $q{=}0.95$      & 20.0     &  1.92 &  15.16 & 27.80 & 0.7746 \\
    Tsallis $q{=}1.5$       & $\infty$ &  0.60 &   2.65 &  4.51 & 0.7552 \\
    \emph{R\'enyi $q{=}0.5$} & \emph{$\infty$} & \emph{2.48} & \emph{10.79} & \emph{17.88} & \emph{0.7765} \\
    \bottomrule
  \end{tabular}
\end{table}

\subsection{The retained information is usable --- on a task the objective never saw}

Held-out AUC asks only whether $z_i^\top z_j$ ranks edges above non-edges, which
is what the decoder was trained to do; it is close to a training metric. We
therefore freeze the posterior mean and probe it for \emph{node class}, a label
that appears nowhere in the objective, the split, or the encoder input, using a
linear classifier fitted on half the nodes and scored on the other half
(Appendix~\ref{app:setup}).

Here the arms separate clearly, and in the direction the collapse result
predicts. On CiteSeer at $c=100$ the VGAE's probe reaches macro-F1 $0.327$ while
Tsallis $q{=}0.5$ reaches $0.427$; at $c=10^3$ the gap widens to $0.278$ against
$0.421$. On Cora at $c=100$ the figures are $0.452$ against $0.529$. The R\'enyi
control again tracks the VGAE rather than the Tsallis arm of the same order
($0.322$ on CiteSeer, $0.453$ on Cora), so the same bound explains both results.
Clustering NMI moves with it (CiteSeer $0.037 \to 0.067$). The negative control
stays negative: on Erd\H{o}s--R\'enyi with random labels every arm sits at chance
through $c\le10^3$ ($\text{F1} \approx 0.16$ for six classes,
$\text{NMI}\le0.03$). At $c=10^4$ every arm has degenerated and macro-F1 falls
below chance.

Two qualifications. On the SBM the ordering \emph{reverses} at moderate weight
($0.953$ for the VGAE against $0.919$ at $c=10$), where the KL acts as a useful
bottleneck; GGVA regains the advantage only past the VGAE's collapse. On PubMed
all arms sit at $\approx0.80$ throughout. The effect appears where the VGAE's
embedding degrades and is absent where it does not.

\section{Discussion}

\textbf{Choosing $q$.} Below $q \approx 0.7$ the penalty sits at its ceiling for
every order (Figure~\ref{fig:qresponse}), making those settings mutually
indistinguishable and close to an unregularized autoencoder. The usable range is
$q \in [0.7, 0.95]$. We use $q=0.5$ for the controlled contrast because it makes
the boundedness mechanism maximally visible; it is a demonstration setting, not
our default recommendation for deployment.

\textbf{Bounded does not mean collapse-proof.} The bound saturates only once the
posterior is already far from the prior. Training \emph{starts} at the prior,
where the penalty is unsaturated and behaves like the KL, so a large enough
weight pins the posterior there from the first epoch and it never reaches the
saturated region. On nine of the ten graphs every arm has crossed below one nat
per node by $c=10^4$; on PubMed none has, even there. At the decade-spaced
resolution of our sweep, $q{=}0.5$ shifts any observed crossing by at most one
grid step. What $q$ buys is delayed collapse, not immunity.

\textbf{Interpreting $c$ across graphs.} Because $\lambda = c/N$, a fixed $c$ is a
different absolute weight on different graphs ($0.33$ on the SBM, $0.005$ on
PubMed). The per-graph ordering in $q$ holds throughout, but
Table~\ref{tab:collapse} is not a comparison at equal regularization pressure;
the swept figures are.

\textbf{Scope of the scalar intervention.} The order has also been made
learnable inside graph attention, where a per-edge $q$ sparsifies neighborhood
weights \citep{dacosta2026sparse}. There $q$ shapes the forward pass; here it
shapes only the latent penalty and leaves message passing untouched.

\textbf{Limitations.} The accuracy result is a null --- no arm separates under a
Friedman test at ten tasks --- which is weaker than showing no advantage exists.
Nulls of this shape are common at this scale: canonical node-level encoders also
fall inside a single Nemenyi clique when benchmarked with the same machinery
\citep{dacosta2026graphnetz}.
All runs use one encoder, decoder and latent width, leaving the $q$--$d$
interaction Proposition~\ref{prop:bound} predicts untested. And the probe is a
linear read-out, not an application: it shows the retained information is
recoverable, not that a practitioner would recover it. The weight sweep has one
point per decade, so it brackets rather than precisely locates collapse
thresholds; a denser grid is needed for sub-decade claims. We report AUC, while
the original VGAE also reports average precision \citep{kipf2016variational};
adding AP is an important extension of the link-prediction null.

\section{Conclusion}

A bounded Tsallis divergence in place of the KL gives one scalar $q$ that
controls posterior collapse by large, consistently ordered factors, with the
unbounded R\'enyi member at the same order isolating boundedness as the cause. A
linear probe recovers node class from the GGVA embedding where the VGAE's has
degraded, but no gain reaches link-prediction accuracy on any of the six larger real graphs: the
task the mechanism was expected to help is not the task it helps.

\bibliographystyle{plainnat}
\bibliography{references}

\newpage
\appendix

\section{Numerical implementation}
\label{app:numerics}

\paragraph{Conditioning near $q=1$.} Written literally, $s = (1-q)\sigma^2 + q$
subtracts two nearly equal quantities whenever $q$ is near $1$. In
\texttt{float32} the cancellation corrupts the last bits of $s$, and $\log s$ is
then wrong by a relative amount that the $1/(2(q-1))$ factor divides by something
small. At the prior itself, where the divergence must be \emph{exactly} zero,
$q=1.05$ evaluated to $2\times10^{-6}$. We instead use
$s = 1 + u$ with $u = (1-q)\,\mathrm{expm1}(2\log\sigma)$ and evaluate
$\log s$ as $\mathrm{log1p}(u)$, which is exact at $\log\sigma = 0$. This matters
because the near-prior regime is precisely where a collapse diagnostic must be
trustworthy.

\paragraph{Validation.} The closed form \eqref{eq:renyi-gauss} and the Tsallis
transform \eqref{eq:tsallis-joint} are checked against numerical quadrature of
the escort integral over $q \in [0.2, 2.0]$; worst absolute
error $5\times10^{-15}$. Quadrature shares none of the Gaussian algebra, so
agreement rules out a sign or factor error rather than re-deriving it.

\paragraph{Divergent and saturating regimes.} For $q>1$ the escort integral
converges only while $\sigma^2 < q/(q-1)$; outside that range $\Ren_q$ is
genuinely infinite and we return $+\infty$ rather than clamping, so a run fails
visibly. The exponent in \eqref{eq:tsallis-joint} is capped at $80$, below
\texttt{float32} overflow, so a diverging $q>1$ run yields a large finite loss
rather than NaNs that hide when the objective broke.

\section{Experimental setup}
\label{app:setup}

Encoder: two-layer variational GCN \citep{kipf2017semi}, hidden width $32$,
latent width $d=16$, with the posterior heads implemented as parallel graph
convolutions. Decoder: parameter-free inner product. Optimizer: Adam, learning
rate $10^{-2}$, $200$ epochs ($300$ on karate). Edge split: $85/5/10$ train /
validation / test via PyTorch Geometric's \texttt{RandomLinkSplit}
\citep{fey2019fast}, \textbf{redrawn per seed}, so the
reported spread covers both initialization and which edges were held out.
Message passing uses training edges only, so no held-out edge enters the
embeddings. Featureless graphs receive identity features (the convention of
\citealt{kipf2016variational}) when small, log-degree features otherwise. Five
seeds ($0$--$4$). All runs on CPU: MPS and CUDA change floating-point reduction
order, so a figure regenerated elsewhere would not be bit-identical.

\paragraph{Metrics.} Test AUC is read at the epoch maximizing validation AUC.
Active units are the fraction of latent dimensions whose $\mu$ varies across
nodes by more than $0.01$. Effective rank is the participation ratio
$(\sum_k \gamma_k)^2/\sum_k \gamma_k^2$ of the latent covariance spectrum, where
$\gamma_k$ are the covariance eigenvalues. It is a
continuous count of used directions that, unlike a thresholded unit count,
distinguishes ``sixteen dimensions carrying a little'' from ``one carrying
everything''.

\section{Complete results}
\label{app:full}

\begin{figure}[htbp]
  \centering
  \includegraphics[width=0.74\textwidth]{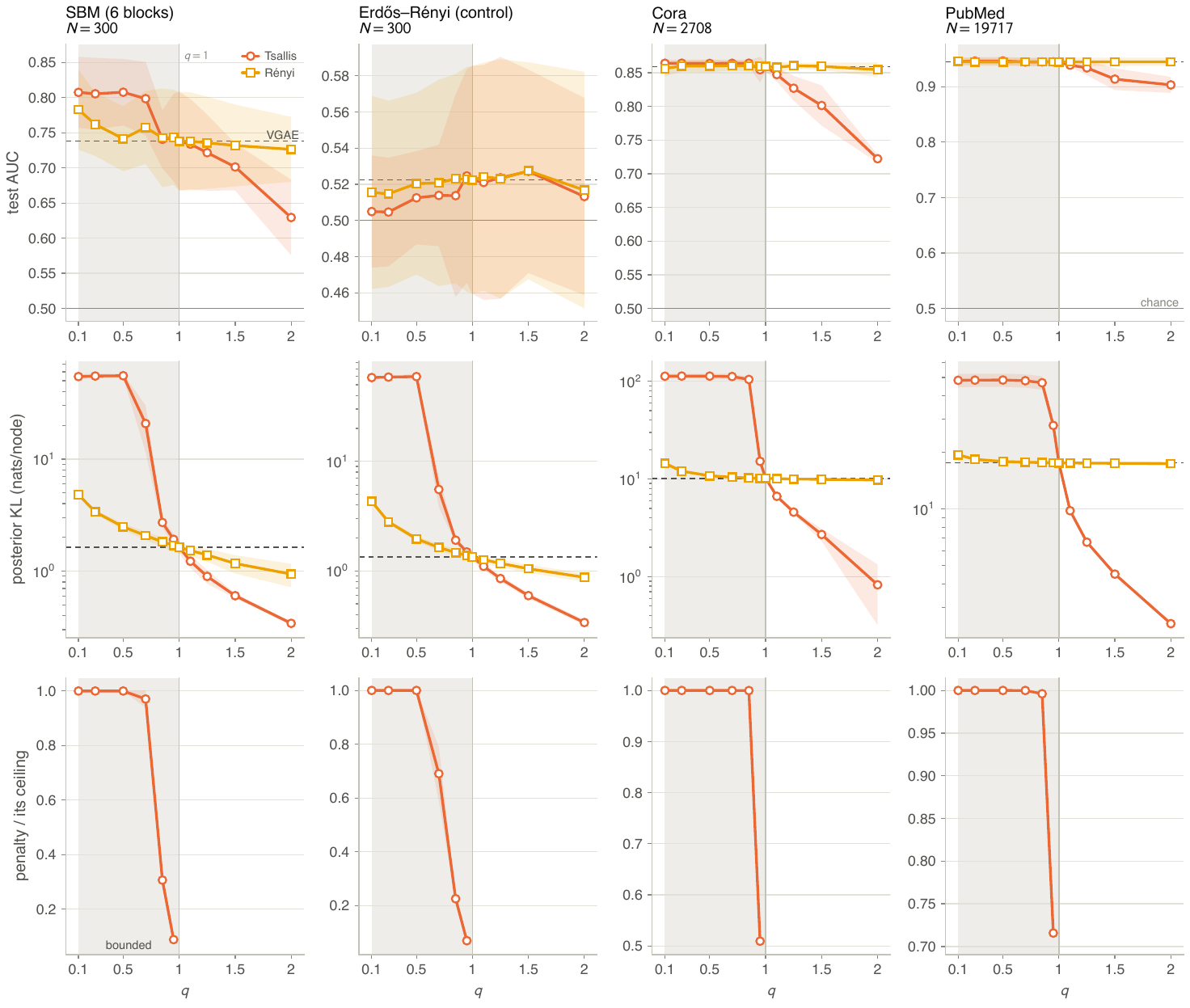}
  \caption{\textbf{The response in $q$, with the mechanism measured.} Held-out
  AUC (top), retained posterior information (middle), and the trained penalty as
  a fraction of its $1/(1-q)$ ceiling (bottom). The dashed line is the VGAE, which
  is the $q{=}1$ point of the same curves. The shaded region is where the Tsallis
  penalty is bounded; the saturation row exists only there, because an unbounded
  penalty has no ceiling to be a fraction of. Below $q \approx 0.7$ the penalty is
  pinned at its ceiling and the orders become indistinguishable.}
  \label{fig:qresponse}
\end{figure}

\begin{figure}[htbp]
  \centering
  \includegraphics[width=\textwidth]{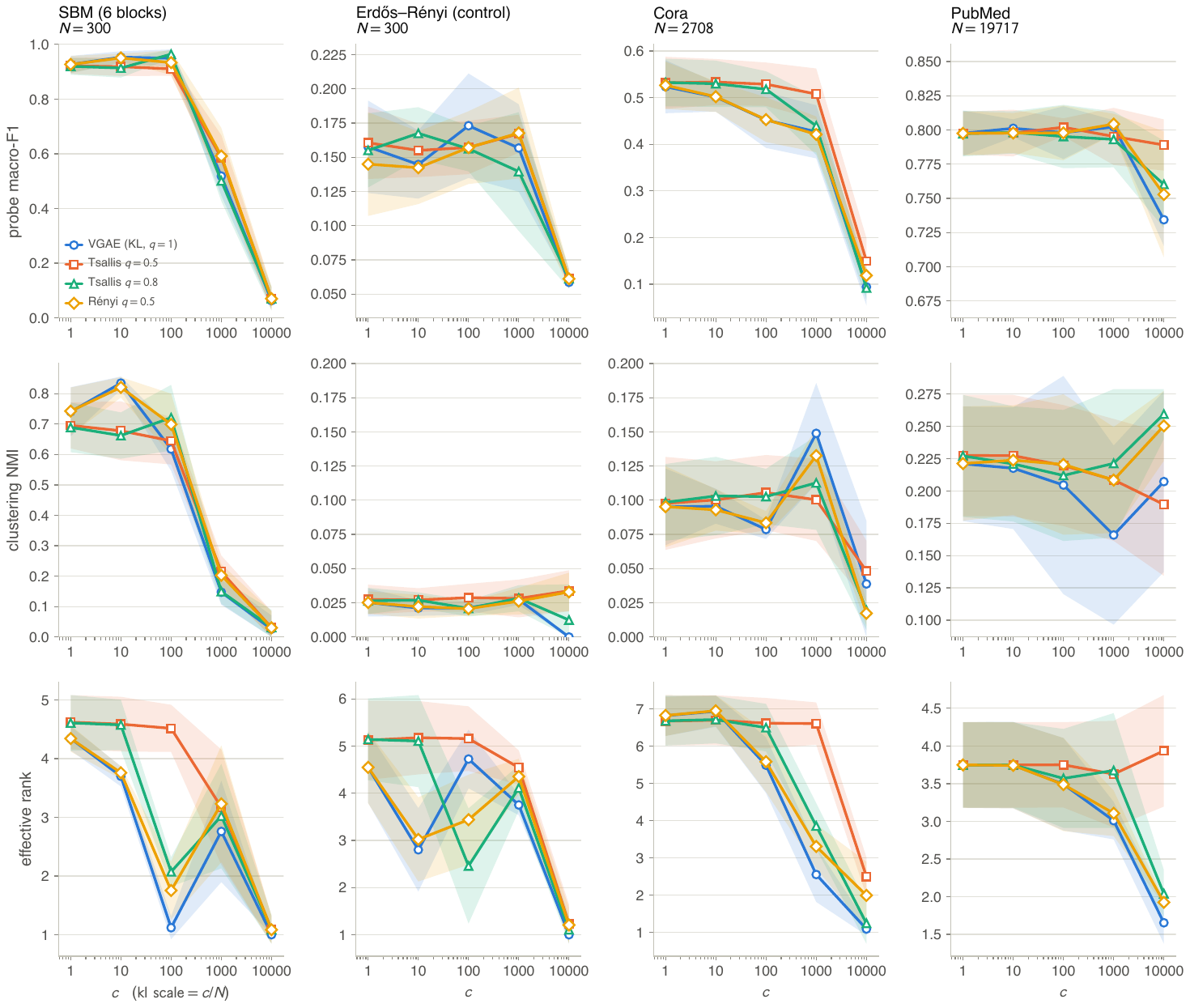}
  \caption{\textbf{Frozen-embedding probes.} Linear-probe macro-F1 (top),
  clustering NMI (middle) and effective rank (bottom) against the penalty weight
  $c$. Node class is never seen by any model. On the citation networks the VGAE's
  probe degrades as $c$ grows while Tsallis $q{=}0.5$ holds; the R\'enyi control
  at the same order tracks the VGAE. The Erd\H{o}s--R\'enyi panel carries random
  labels. Every arm sits at chance through $c\le10^3$; at $c=10^4$ all arms
  degenerate and macro-F1 falls below chance.}
  \label{fig:probe}
\end{figure}

\input{tables/exp01_generalized_divergence.tex}

\begin{figure}[p]
  \centering
  \includegraphics[width=0.98\textwidth]{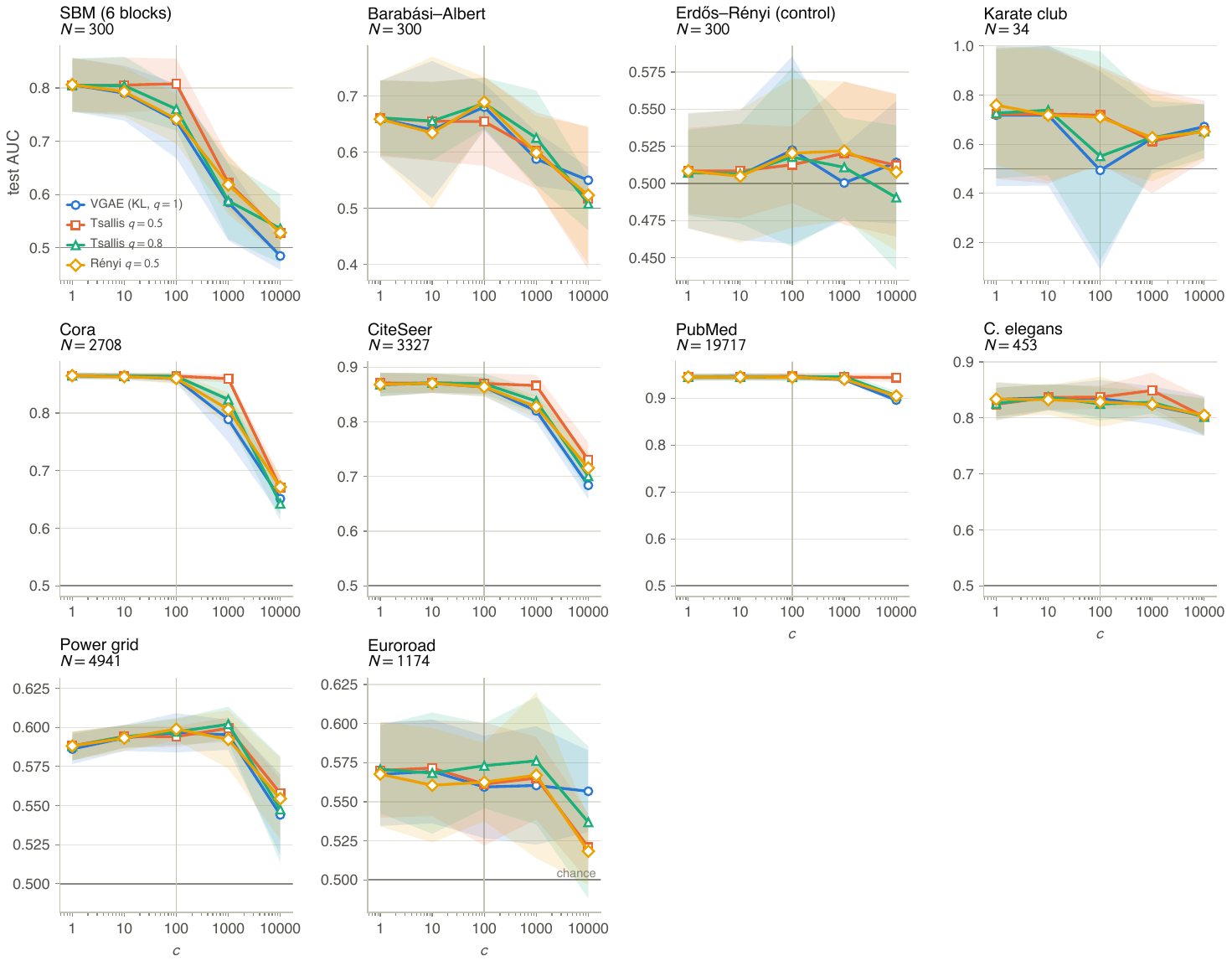}
  \caption{\textbf{Held-out AUC over the complete ten-graph weight sweep.}
  This is the full-suite counterpart of the top row of Figure~\ref{fig:sweep}.}
  \label{fig:sweep-auc-all}
\end{figure}

\begin{figure}[p]
  \centering
  \includegraphics[width=0.98\textwidth]{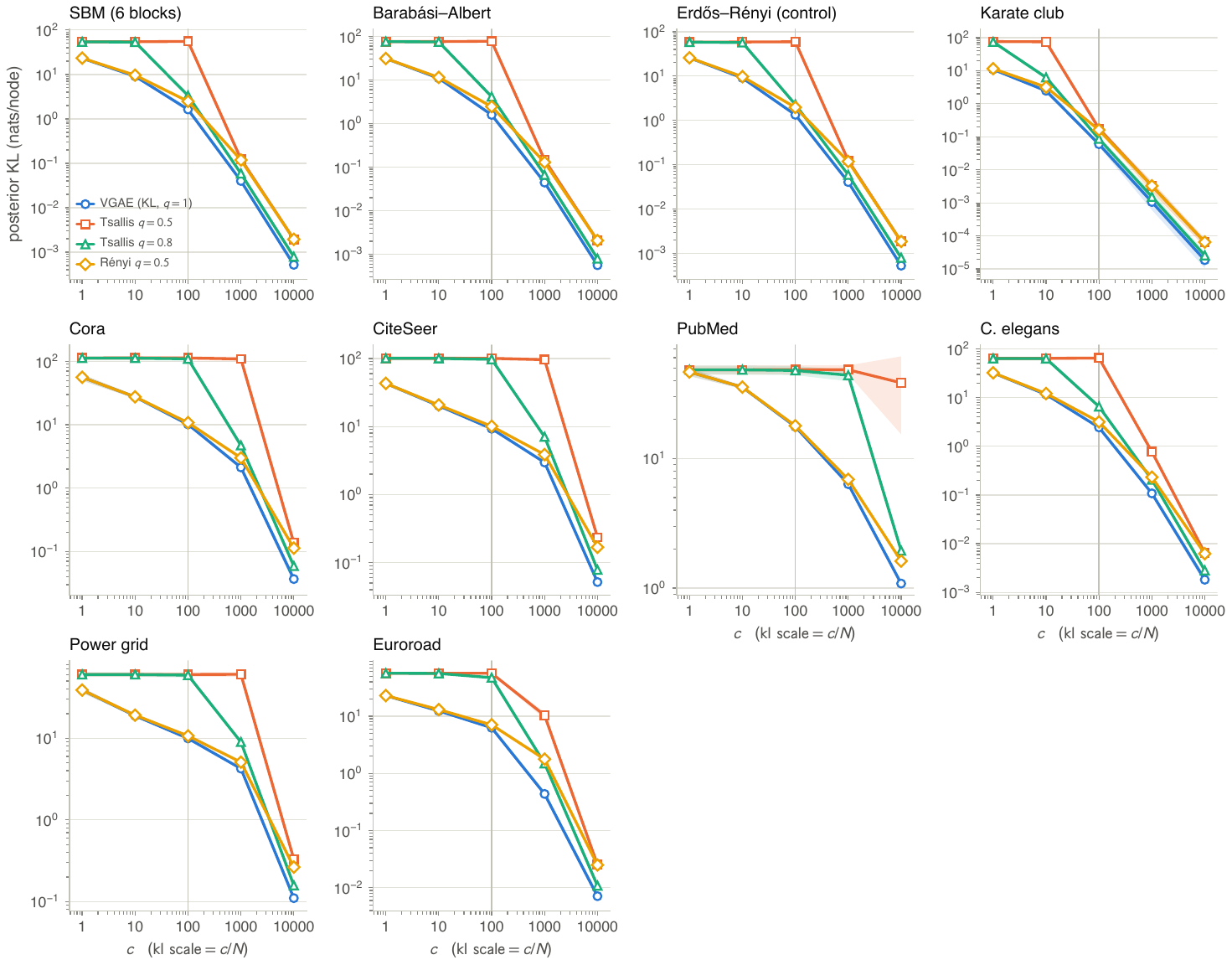}
  \caption{\textbf{Retained posterior KL over the complete ten-graph weight
  sweep.} This is the full-suite counterpart of the bottom row of
  Figure~\ref{fig:sweep}.}
  \label{fig:sweep-kl-all}
\end{figure}

\begin{figure}[p]
  \centering
  \includegraphics[width=0.98\textwidth]{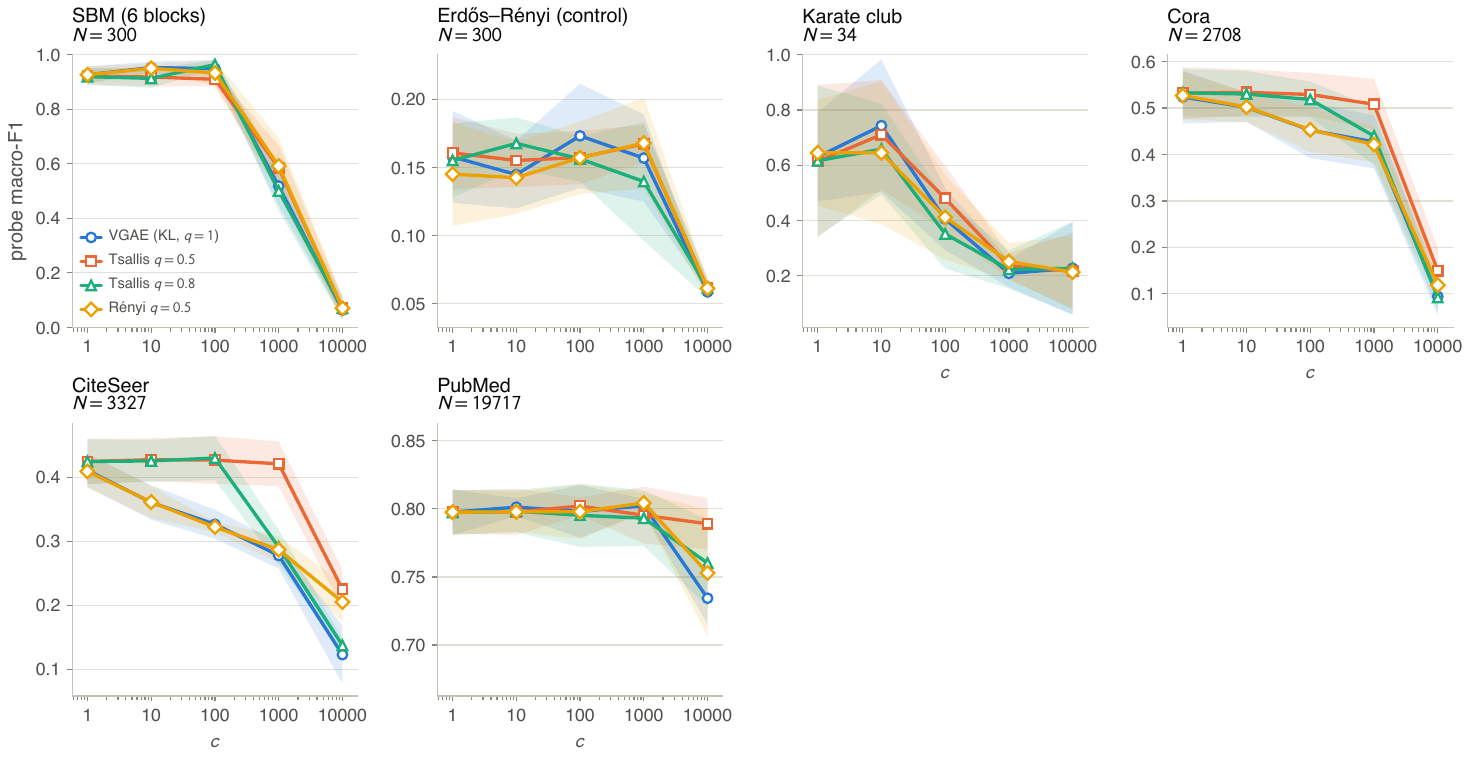}
  \includegraphics[width=0.98\textwidth]{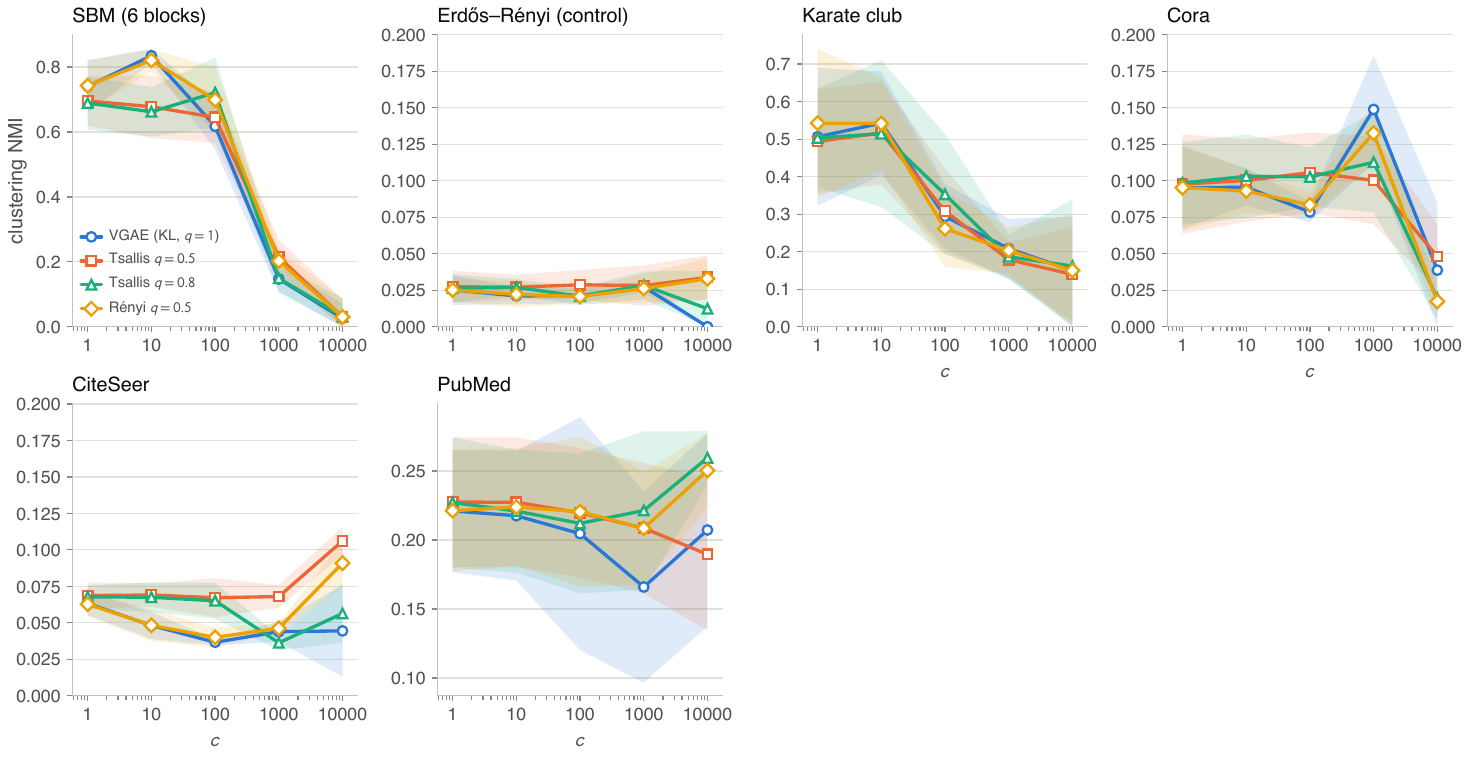}
  \caption{\textbf{Frozen-embedding label probes over all ten graphs.}
  Macro-F1 (top) and clustering NMI (bottom) extend Figure~\ref{fig:probe} to
  the full suite. Labels are unavailable on graphs where a probe is undefined.}
  \label{fig:probe-label-all}
\end{figure}

\begin{figure}[p]
  \centering
  \includegraphics[width=0.98\textwidth]{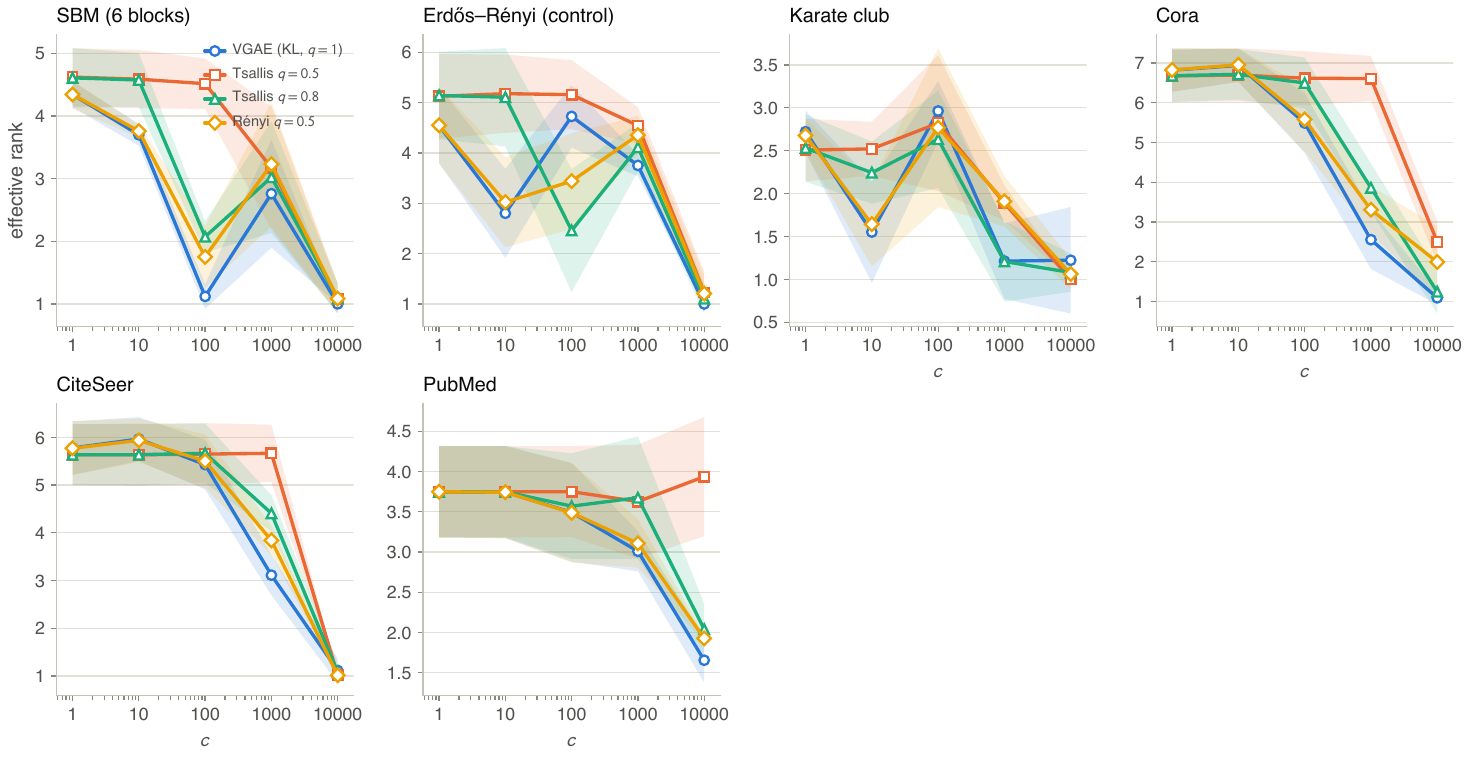}
  \caption{\textbf{Effective rank over the complete ten-graph sweep.}
  This continuous diagnostic complements the thresholded active-unit count.}
  \label{fig:probe-rank-all}
\end{figure}

Every per-seed number behind every figure is written to JSON alongside the
figures, including both epoch-selection protocols so that the choice of
\texttt{best\_val} over \texttt{final} is auditable rather than hidden.

\paragraph{Retained information relative to the VGAE at $c=100$.} Ratio of KL
nats, Tsallis $q{=}0.5$ / VGAE and R\'enyi $q{=}0.5$ / VGAE respectively: SBM
$34.0\times$ / $1.52\times$; Barab\'asi--Albert $49.0\times$ / $1.55\times$;
Erd\H{o}s--R\'enyi $44.4\times$ / $1.47\times$; karate $3.0\times$ / $2.70\times$;
Cora $11.2\times$ / $1.06\times$; CiteSeer $10.8\times$ / $1.09\times$; PubMed
$2.8\times$ / $1.02\times$; \emph{C.\ elegans} $26.6\times$ / $1.30\times$; power
grid $6.1\times$ / $1.07\times$; Euroroad $9.0\times$ / $1.13\times$. The
separation between the two columns is the isolation of the mechanism.

\paragraph{Aggregate ranking.} We follow the Friedman/Nemenyi protocol for
comparing methods over multiple datasets \citep{demsar2006statistical}, as
applied to graph neural network benchmarking by \citet{dacosta2026graphnetz}.
Friedman over ten tasks and six arms:
$\chi^2 = 7.26$, $p = 0.20$, not rejected. Mean ranks (lower is better): Tsallis
$q{=}0.5$ $2.80$; Tsallis $q{=}0.8$ $3.10$; R\'enyi $q{=}0.5$ $3.10$; Tsallis
$q{=}0.95$ $3.20$; VGAE $4.30$; Tsallis $q{=}1.5$ $4.50$. Nemenyi critical
difference at $\alpha=0.05$ is $2.384$ against an observed spread of $1.70$, so no
pair is separable. The ordering is suggestive and not significant, and eight
tasks would not have been enough either.

\section{Reproducibility}

The implementation are in the accompanying repository: \href{https://github.com/kleyt0n/ggva}{github.com/kleyt0n/ggva}.

\end{document}

%% file: tables/exp01_generalized_divergence.tex
\begin{table}[htbp]
  \centering
  \caption{\textbf{Held-out AUC for every graph and arm at $c=100$.}
  Members of the Rényi--Tsallis family act as the latent penalty at
  $\mathrm{kl\ scale}=100/N$. Entries are mean $\pm$ standard deviation over
  five seeds. The VGAE is the $q=1$ arm; only the divergence differs. BA is
  Barabási--Albert and ER is the Erdős--Rényi negative control.}
  \label{tab:exp01_generalized_divergence}
  \scriptsize
  \begin{tabular}{lccccc}
    \toprule
    Model & BA & C. elegans & CiteSeer & Cora & ER control \\
    \midrule
    Rényi $q=0.5$   & $\mathbf{0.689 \pm 0.045}$ & $0.829 \pm 0.045$ & $0.864 \pm 0.018$ & $0.860 \pm 0.009$ & $0.520 \pm 0.050$ \\
    Tsallis $q=0.5$ & $0.654 \pm 0.079$ & $0.837 \pm 0.020$ & $\mathbf{0.870 \pm 0.018}$ & $\mathbf{0.864 \pm 0.005}$ & $0.513 \pm 0.026$ \\
    Tsallis $q=0.8$ & $0.687 \pm 0.045$ & $0.825 \pm 0.030$ & $0.870 \pm 0.019$ & $0.864 \pm 0.006$ & $0.518 \pm 0.060$ \\
    Tsallis $q=0.95$ & $0.686 \pm 0.045$ & $0.818 \pm 0.050$ & $0.866 \pm 0.018$ & $0.855 \pm 0.019$ & $0.525 \pm 0.059$ \\
    Tsallis $q=1.5$ & $0.618 \pm 0.050$ & $\mathbf{0.839 \pm 0.033}$ & $0.824 \pm 0.017$ & $0.800 \pm 0.033$ & $\mathbf{0.527 \pm 0.056}$ \\
    VGAE (KL, $q=1$) & $0.680 \pm 0.042$ & $0.835 \pm 0.033$ & $0.864 \pm 0.017$ & $0.859 \pm 0.007$ & $0.522 \pm 0.063$ \\
    \bottomrule
  \end{tabular}

  \medskip
  \begin{tabular}{lccccc}
    \toprule
    Model & Euroroad & Karate & Power grid & PubMed & SBM \\
    \midrule
    Rényi $q=0.5$   & $0.563 \pm 0.025$ & $0.710 \pm 0.199$ & $0.599 \pm 0.007$ & $0.945 \pm 0.007$ & $0.742 \pm 0.046$ \\
    Tsallis $q=0.5$ & $0.561 \pm 0.039$ & $\mathbf{0.718 \pm 0.198}$ & $0.594 \pm 0.007$ & $\mathbf{0.946 \pm 0.007}$ & $\mathbf{0.808 \pm 0.047}$ \\
    Tsallis $q=0.8$ & $\mathbf{0.573 \pm 0.027}$ & $0.551 \pm 0.427$ & $0.597 \pm 0.008$ & $0.945 \pm 0.007$ & $0.760 \pm 0.040$ \\
    Tsallis $q=0.95$ & $0.564 \pm 0.021$ & $0.539 \pm 0.501$ & $\mathbf{0.600 \pm 0.008}$ & $0.945 \pm 0.008$ & $0.743 \pm 0.062$ \\
    Tsallis $q=1.5$ & $0.561 \pm 0.024$ & $0.563 \pm 0.343$ & $0.593 \pm 0.010$ & $0.914 \pm 0.020$ & $0.701 \pm 0.033$ \\
    VGAE (KL, $q=1$) & $0.559 \pm 0.033$ & $0.494 \pm 0.402$ & $0.597 \pm 0.013$ & $0.945 \pm 0.007$ & $0.738 \pm 0.070$ \\
    \bottomrule
  \end{tabular}
\end{table}

%% file: references.bib
@inproceedings{kingma2014auto,
  title     = {Auto-Encoding Variational {B}ayes},
  author    = {Kingma, Diederik P. and Welling, Max},
  booktitle = {International Conference on Learning Representations (ICLR)},
  year      = {2014},
  url       = {https://arxiv.org/abs/1312.6114}
}

@article{kipf2016variational,
  title   = {Variational Graph Auto-Encoders},
  author  = {Kipf, Thomas N. and Welling, Max},
  journal = {NeurIPS Workshop on Bayesian Deep Learning},
  year    = {2016},
  url     = {https://arxiv.org/abs/1611.07308}
}

@inproceedings{kipf2017semi,
  title     = {Semi-Supervised Classification with Graph Convolutional Networks},
  author    = {Kipf, Thomas N. and Welling, Max},
  booktitle = {International Conference on Learning Representations (ICLR)},
  year      = {2017},
  url       = {https://arxiv.org/abs/1609.02907}
}

@article{tsallis1988possible,
  title   = {Possible generalization of {B}oltzmann--{G}ibbs statistics},
  author  = {Tsallis, Constantino},
  journal = {Journal of Statistical Physics},
  volume  = {52},
  number  = {1--2},
  pages   = {479--487},
  year    = {1988},
  doi     = {10.1007/BF01016429},
  url     = {https://doi.org/10.1007/BF01016429}
}

@inproceedings{renyi1961measures,
  title     = {On Measures of Entropy and Information},
  author    = {R{\'e}nyi, Alfr{\'e}d},
  booktitle = {Proceedings of the Fourth Berkeley Symposium on Mathematical
               Statistics and Probability},
  volume    = {1},
  pages     = {547--561},
  year      = {1961},
  url       = {https://projecteuclid.org/euclid.bsmsp/1200512181}
}

@article{vanerven2014renyi,
  title   = {R{\'e}nyi Divergence and {K}ullback--{L}eibler Divergence},
  author  = {van Erven, Tim and Harremo{\"e}s, Peter},
  journal = {IEEE Transactions on Information Theory},
  volume  = {60},
  number  = {7},
  pages   = {3797--3820},
  year    = {2014},
  doi     = {10.1109/TIT.2014.2320500},
  url     = {https://arxiv.org/abs/1206.2459}
}

@inproceedings{li2016renyi,
  title     = {R{\'e}nyi Divergence Variational Inference},
  author    = {Li, Yingzhen and Turner, Richard E.},
  booktitle = {Advances in Neural Information Processing Systems (NeurIPS)},
  year      = {2016},
  url       = {https://arxiv.org/abs/1602.02311}
}

@article{kobayashi2020qvae,
  title   = {q-{VAE} for Disentangled Representation Learning and Latent
             Dynamical Systems},
  author  = {Kobayashi, Taisuke},
  journal = {IEEE Robotics and Automation Letters},
  volume  = {5},
  number  = {4},
  pages   = {5669--5676},
  year    = {2020},
  doi     = {10.1109/LRA.2020.3010206},
  url     = {https://arxiv.org/abs/2003.01852}
}

@inproceedings{khan2021epitomic,
  title     = {Epitomic Variational Graph Autoencoder},
  author    = {Khan, Rayyan Ahmad and Anwaar, Muhammad Umer and
               Kleinsteuber, Martin},
  booktitle = {25th International Conference on Pattern Recognition (ICPR)},
  pages     = {7203--7210},
  year      = {2021},
  doi       = {10.1109/ICPR48806.2021.9412531},
  url       = {https://arxiv.org/abs/2004.01468}
}

@inproceedings{zhao2019infovae,
  title     = {{InfoVAE}: Balancing Learning and Inference in Variational
               Autoencoders},
  author    = {Zhao, Shengjia and Song, Jiaming and Ermon, Stefano},
  booktitle = {AAAI Conference on Artificial Intelligence},
  volume    = {33},
  pages     = {5885--5892},
  year      = {2019},
  doi       = {10.1609/aaai.v33i01.33015885},
  url       = {https://ojs.aaai.org/index.php/AAAI/article/view/4538}
}

@inproceedings{tolstikhin2018wasserstein,
  title     = {Wasserstein Auto-Encoders},
  author    = {Tolstikhin, Ilya and Bousquet, Olivier and Gelly, Sylvain and
               Sch{\"o}lkopf, Bernhard},
  booktitle = {International Conference on Learning Representations (ICLR)},
  year      = {2018},
  url       = {https://openreview.net/forum?id=HkL7n1-0b}
}

@inproceedings{kim2024t3vae,
  title     = {{$t^3$}-Variational Autoencoder: Learning Heavy-Tailed Data with
               {S}tudent's $t$ and Power Divergence},
  author    = {Kim, Juno and Kwon, Jaehyuk and Cho, Mincheol and Lee, Hyunjong
               and Won, Joong-Ho},
  booktitle = {International Conference on Learning Representations (ICLR)},
  year      = {2024},
  url       = {https://openreview.net/forum?id=RzNlECeoOB}
}

@inproceedings{higgins2017beta,
  title     = {beta-{VAE}: Learning Basic Visual Concepts with a Constrained
               Variational Framework},
  author    = {Higgins, Irina and Matthey, Loic and Pal, Arka and Burgess,
               Christopher and Glorot, Xavier and Botvinick, Matthew and
               Mohamed, Shakir and Lerchner, Alexander},
  booktitle = {International Conference on Learning Representations (ICLR)},
  year      = {2017},
  url       = {https://openreview.net/forum?id=Sy2fzU9gl}
}

@inproceedings{bowman2016generating,
  title     = {Generating Sentences from a Continuous Space},
  author    = {Bowman, Samuel R. and Vilnis, Luke and Vinyals, Oriol and Dai,
               Andrew M. and J{\'o}zefowicz, Rafal and Bengio, Samy},
  booktitle = {Conference on Computational Natural Language Learning (CoNLL)},
  year      = {2016},
  doi       = {10.18653/v1/K16-1002},
  url       = {https://arxiv.org/abs/1511.06349}
}

@inproceedings{sonderby2016ladder,
  title     = {Ladder Variational Autoencoders},
  author    = {S{\o}nderby, Casper Kaae and Raiko, Tapani and Maal{\o}e, Lars
               and S{\o}nderby, S{\o}ren Kaae and Winther, Ole},
  booktitle = {Advances in Neural Information Processing Systems (NeurIPS)},
  year      = {2016},
  url       = {https://arxiv.org/abs/1602.02282}
}

@inproceedings{alemi2018fixing,
  title     = {Fixing a Broken {ELBO}},
  author    = {Alemi, Alexander A. and Poole, Ben and Fischer, Ian and
               Dillon, Joshua V. and Saurous, Rif A. and Murphy, Kevin},
  booktitle = {International Conference on Machine Learning (ICML)},
  year      = {2018},
  url       = {https://arxiv.org/abs/1711.00464}
}

@inproceedings{burda2016importance,
  title     = {Importance Weighted Autoencoders},
  author    = {Burda, Yuri and Grosse, Roger B. and Salakhutdinov, Ruslan},
  booktitle = {International Conference on Learning Representations (ICLR)},
  year      = {2016},
  url       = {https://arxiv.org/abs/1509.00519}
}

@article{demsar2006statistical,
  title   = {Statistical Comparisons of Classifiers over Multiple Data Sets},
  author  = {Dem{\v{s}}ar, Janez},
  journal = {Journal of Machine Learning Research},
  volume  = {7},
  pages   = {1--30},
  year    = {2006},
  url     = {https://www.jmlr.org/papers/v7/demsar06a.html}
}

@article{sen2008collective,
  title   = {Collective Classification in Network Data},
  author  = {Sen, Prithviraj and Namata, Galileo and Bilgic, Mustafa and
             Getoor, Lise and Galligher, Brian and Eliassi-Rad, Tina},
  journal = {AI Magazine},
  volume  = {29},
  number  = {3},
  pages   = {93--106},
  year    = {2008},
  doi     = {10.1609/aimag.v29i3.2157},
  url     = {https://doi.org/10.1609/aimag.v29i3.2157}
}

@article{fey2019fast,
  title   = {Fast Graph Representation Learning with {P}y{T}orch {G}eometric},
  author  = {Fey, Matthias and Lenssen, Jan Eric},
  journal = {ICLR Workshop on Representation Learning on Graphs and Manifolds},
  year    = {2019},
  url     = {https://arxiv.org/abs/1903.02428}
}

@article{watts1998collective,
  title   = {Collective dynamics of `small-world' networks},
  author  = {Watts, Duncan J. and Strogatz, Steven H.},
  journal = {Nature},
  volume  = {393},
  number  = {6684},
  pages   = {440--442},
  year    = {1998},
  doi     = {10.1038/30918},
  url     = {https://doi.org/10.1038/30918}
}

@article{holland1983stochastic,
  title   = {Stochastic blockmodels: First steps},
  author  = {Holland, Paul W. and Laskey, Kathryn Blackmond and Leinhardt, Samuel},
  journal = {Social Networks},
  volume  = {5},
  number  = {2},
  pages   = {109--137},
  year    = {1983},
  doi     = {10.1016/0378-8733(83)90021-7},
  url     = {https://doi.org/10.1016/0378-8733(83)90021-7}
}

@article{dacosta2026perspectives,
  title   = {Perspectives on {T}sallis Statistics for Artificial Intelligence},
  author  = {da Costa, Kleyton and Modenesi, Bernardo},
  journal = {arXiv preprint arXiv:2608.01223},
  year    = {2026},
  url     = {https://arxiv.org/abs/2608.01223}
}

@article{dacosta2026sparse,
  title   = {When Should Graph Attention Be Sparse? {L}earning a Per-Edge
             {T}sallis Index},
  author  = {da Costa, Kleyton and Modenesi, Bernardo},
  journal = {arXiv preprint arXiv:2608.02938},
  year    = {2026},
  url     = {https://arxiv.org/abs/2608.02938}
}

@article{dacosta2026graphnetz,
  title   = {{G}raph{N}etz: Statistical Benchmarking of Graph Neural Networks
             with Paired Tests and Rank Aggregation},
  author  = {da Costa, Kleyton and Modenesi, Bernardo},
  journal = {arXiv preprint arXiv:2605.09099},
  year    = {2026},
  url     = {https://arxiv.org/abs/2605.09099}
}

@article{barabasi1999emergence,
  title   = {Emergence of Scaling in Random Networks},
  author  = {Barab{\'a}si, Albert-L{\'a}szl{\'o} and Albert, R{\'e}ka},
  journal = {Science},
  volume  = {286},
  number  = {5439},
  pages   = {509--512},
  year    = {1999},
  doi     = {10.1126/science.286.5439.509},
  url     = {https://doi.org/10.1126/science.286.5439.509}
}
